\documentclass[letterpaper, 10 pt, conference]{ieeeconf}
\usepackage{lmodern}
\IEEEoverridecommandlockouts
\title{\LARGE \bf
Residual Learning-Based Control of Vehicle Platoons with $\ell_2$ Stability Guarantees via Recurrent Equilibrium Networks
}

\author{Brian Delgado, Anh-Tu Nguyen, \textit{Senior Member, IEEE}, and Hamid Taghavifar, \textit{Senior Member, IEEE}%
\thanks{Brian Delgado and Hamid Taghavifar are with Concordia University, Montreal, QC, Canada
(e-mail: {\tt\small \{brian.delgado,hamid.taghavifar\}@concordia.ca}).}%
\thanks{Anh-Tu Nguyen is with the Research Center LAMIH UMR CNRS 8201,
Université Polytechnique Hauts-de-France, and also with INSA Hauts-de-France,
Valenciennes, France (e-mail: {\tt\small tnguyen@uphf.fr}).}%
\thanks{\textcolor{blue}{This work has been submitted to the IEEE for possible publication.
Copyright may be transferred without notice, after which this version may no longer be accessible.}}%
}
 
\usepackage{amsmath}   
\usepackage{amssymb}
\usepackage{float}
\usepackage{amsfonts,mathtools}
\usepackage{bm}
\usepackage[T1]{fontenc}
\usepackage[hidelinks]{hyperref}
 \usepackage{xcolor}

\newtheorem{theorem}{Theorem}
\newtheorem{definition}{Definition}
\newtheorem{proposition}{Proposition}
 
\usepackage{xcolor}

\newcommand{\normtwo}[1]{\lVert #1 \rVert_{2}}
\newcommand{\normld}[1]{\lVert #1 \rVert_{\ell_2}}
\newcommand{\normH}[1]{\lVert #1 \rVert_{\mathcal H_\infty}}

\usepackage{graphicx}
\graphicspath{{figures/}}

\begin{document}

\maketitle
\thispagestyle{empty}
\pagestyle{empty}

\begin{abstract}
This paper proposes a residual learning-based control framework for heterogeneous vehicle platoons subject to parametric uncertainty and external disturbances. A nominal controller designed via Linear Matrix Inequalities (LMIs), along with disturbance-observer compensation, is enhanced by a Recurrent Equilibrium Network (REN) trained offline using stored trajectories and nominal-model prediction errors. The REN is constrained to satisfy a prescribed \(\ell_2\)-gain bound, enabling sufficient small-gain conditions for local closed-loop stability and disturbance string stability. Experiments demonstrate reduced spacing and velocity errors relative to the nominal controller.
\end{abstract}
\section{Introduction}
Vehicle platooning can improve road capacity, safety, traffic flow, and energy efficiency \cite{Gaagai2023}. Under a predecessor–follower (PF) communication topology, each follower relies only on information from its predecessor, which reduces the communication burden but makes string stability more challenging to achieve \cite{Mahfouz2023}. String stability \cite{Feng2019} prevents disturbance amplification along the vehicle string. To guarantee string stability, various platoon control strategies have been proposed, including $H_\infty$ control \cite{Ploeg2014}, adaptive control \cite{Li2023}, and model predictive control \cite{Luo20212}.

Beyond closed-loop stability, many applications also require satisfactory performance. 
Recent studies have explored Reinforcement Learning (RL) controllers to improve performance \cite{Furieri2024}, typically using deep neural networks (DNNs) to approximate complex nonlinear mappings \cite{Degrave2022}. Although RL and deep learning offer significant potential for improving the performance of nonlinear control systems, their adoption in safety-critical applications remains limited by the difficulty of establishing formal stability and robustness guarantees. This limitation is particularly important in autonomous mobility systems, where safety and reliable operation are essential requirements.

Recent approaches to learning provably stabilizing DNN controllers can be broadly divided into two groups. The first consists of constrained optimization approaches \cite{Min2023}, which enforce Lyapunov-like inequalities during training. However, these constraints can be overly conservative, thereby restricting the set of admissible policies. Moreover, enforcing them during optimization can be computationally expensive, which may limit practical applicability. The second group addresses stability by using classes of control policies with built-in stability guarantees \cite{Furieri2024}, for which stability follows directly from the policy structure.

Motivated by these results, this paper addresses the design of a neural-network-enhanced controller for vehicle platooning under a PF communication topology. Following the second class of approaches, we exploit the recently developed properties of Recurrent Equilibrium Networks (RENs) \cite{Revay2024} to parameterize stable controllers and obtain controllers that are stable by design. The main contributions of this work are summarized as follows:
(i) a residual-learning-based platoon controller that combines a nominal controller, a disturbance observer, and an REN-based compensation policy;
(ii) an offline-trained REN that improves performance in the presence of model mismatch while preserving a prescribed $\ell_2$-gain bound; and
(iii) sufficient small-gain conditions that guarantee individual-vehicle stability and string stability of the interconnected platoon.

\noindent 
\textit{Notation.}
For a square matrix $X$, $X\succ0$ denotes positive definiteness. The symbol $*$ stands for matrix blocks that can be deduced by symmetry. $\|\cdot\|_2$ denotes the Euclidean norm for vectors and the spectral
norm for matrices, while for a discrete-time sequence
$x=\{x_k\}_{k\geq0}$,
$
\|x\|_{\ell_2}:=\left(\sum_{k=0}^{\infty}\|x_k\|_2^2\right)^{1/2}.
$
For a stable transfer matrix $G$, $\|G\|_{\mathcal H_\infty}$ denotes
its induced $\ell_2$ gain. $\operatorname{blkdiag}(\cdot)$ denotes
block-diagonal concatenation.

\section{Problem Formulation}

This section formalizes the platooning control problem.

\subsection{Vehicle Dynamics and Distance Policy}
Consider a platoon of $N+1$ vehicles traveling with the PF communication topology, with the leader indexed by $i=0$. The $i$th following vehicle dynamics are given by \cite{Luo20212}:
\begin{equation}
\begin{aligned}
  p_{i,k+1}  &= p_{i,k} + v_{i,k}T \\
  v_{i,k+1} &= v_{i,k}+ a_{i,k}T \\
  a_{i,k+1} &= a_{i,k} + f_i(v_{i,k},a_{i,k}) +  B_iu_{i,k} + D_iF_{r,i,k}
\end{aligned}
\label{vehi_dyn}
\end{equation}
with
\[
\begin{aligned}
&B_i = \frac{T}{m_i \tau_i},~~ D_i = \frac{g T}{\tau_i}\\
&f_i(v_{i,k},a_{i,k}) = -\frac{T}{ \tau_i}a_{i,k} -\frac{Tc_i}{m_i \tau_i} 
    \left(v_{i,k}^2 + 2\tau_i v_{i,k}a_{i,k}\right)
\end{aligned}
\]
where $p_{i,k}$, $v_{i,k}$ and $a_{i,k}$ denote the position, velocity and acceleration of the $i$th vehicle, respectively, $\tau_i$ represents the inertial delay time constant, $m_i$ the vehicle mass, $c_i$ the aerodynamic drag coefficient and T the sampling period, $u_{i,k}$ is the control input representing the desired longitudinal force, and $F_{r,i,k}$ is an unknown disturbance. 
The desired inter-vehicle distance is specified by the spacing policy $d_{r,i}= r_i+hv_{i,k}$ \cite{Silva2025}, where $d_{r,i}$ denotes the desired inter-vehicle distance, $r_i$ is a constant representing the standstill distance and the $i$th vehicle length, $h$ is the time gap, and $v_{i,k}$ is the velocity of the $i$th vehicle. Based on the vehicle dynamics \eqref{vehi_dyn} and the spacing policy, the spacing policy error between two consecutive vehicles is defined as the difference between the actual inter-vehicle distance and the desired distance as
\begin{equation}
    \Delta d_{i,k} = p_{i-1,k} - p_{i,k} - d_{r,i} = \Delta p_{i,k} - d_{r,i}
    \label{distance_error}
\end{equation}
\noindent where $\Delta p_{i,k}$ denotes the relative distance between vehicles.

\subsection{Disturbance Observer Compensation}

The vehicle dynamics in \eqref{vehi_dyn} may be affected by parametric uncertainties and disturbances. Since only nominal vehicle parameters are assumed to be available for control design, the following nominal disturbance-free model is formulated
\begin{equation}
    a_{n,i,k+1} = a_{i,k} + f_{n,i}(v_{i,k},a_{i,k}) + B_{n,i}u_{i,k}
    \label{nom_eq}
\end{equation}
\noindent where  $f_{n,i}(v_{i,k},a_{i,k})$ is computed using $f_{i}(v_{i,k},a_{i,k})$ with the nominal vehicle parameters. Assuming bounded parametric uncertainties, the real vehicle dynamics \eqref{vehi_dyn} can be rewritten in terms of the nominal model \eqref{nom_eq} as
\begin{equation}
    a_{i,k+1} = a_{i,k} + f_{n,i}(v_{i,k},a_{i,k}) + B_{n,i}u_{i,k} + \delta_{i,k}
    \label{nom_m}
\end{equation}
\noindent where $\delta_{i,k}$ is the virtual disturbance, defined as $
    \delta_{i,k} = f_i(v_{i,k},a_{i,k}) - f_{n,i}(v_{i,k},a_{i,k}) + (B_i - B_{n,i})u_{i,k} + D_iF_{r,i,k}$.
Based on \eqref{nom_m}, the  following disturbance observer is introduced to estimate and compensate for the virtual disturbance:
\begin{equation}
\begin{aligned}
    \varsigma_{i,k+1} &= \varsigma_{i,k} + l_d \left( f_{n,i}(v_{i,k},a_{i,k}) + B_{n,i} u_{i,k} + \hat{\delta}_{i,k} \right) \\
    \hat{\delta}_{i,k} &= l_d a_{i,k} - \varsigma_{i,k}
\end{aligned}
\label{dist_obs}
\end{equation}
\noindent where $\hat{\delta}_{i,k}$ denotes the estimated disturbance, $\varsigma_{i,k}$ is an internal state of the disturbance observer, and $l_d$ is the observer gain. 
Defining the disturbance estimation error as $\tilde{\delta}_{i,k} = \delta_{i,k} - \hat{\delta}_{i,k}$, 
the disturbance estimation error dynamics obtained from \eqref{nom_m} and \eqref{dist_obs} 
are given by
$
   \tilde{\delta}_{i,k+1} = (1 - l_d)\tilde{\delta}_{i,k}
   + \delta_{i,k+1}-\delta_{i,k}.
$
Assume $\Delta\delta_i\in\ell_2$, where
$\Delta\delta_{i,k}=\delta_{i,k+1}-\delta_{i,k}$.
Therefore, selecting $0<l_d<2$ ensures that the homogeneous disturbance-estimation error dynamics are Schur stable and $\tilde{\delta}_i\in\ell_2$ .
Given the disturbance observer in \eqref{dist_obs}, defining $\beta_f=\beta-1$,for $B_{n,i} \neq 0$, the following feedback-linearizing control law is used:
\begin{equation}
    u_{i,k} \!=\! \frac{1}{B_{n,i}} \left[
    \beta_f(a_{i,k}-u_{n,i,k})
    \!-\! f_{n,i}(v_{i,k},a_{i,k})
    \!+\! \mu_{i,k} - \hat{\delta}_{i,k}
    \right]
    \label{flcontrol}
\end{equation}
where $u_{n,i,k}$ is the control input from the  nominal controller, whereas  $\mu_{i,k}$ is the control input from a neural enhanced controller, with $0<\beta<1$. Actuator saturation and rate limits are neglected; the commanded input \(u_i\) is assumed to be applied exactly. Therefore, the stability guarantees apply to the unsaturated closed-loop system.
Substituting \eqref{flcontrol} into the nominally rewritten vehicle dynamics in \eqref{nom_m} yields
\begin{equation}
\begin{aligned}
  a_{i,k+1} &= \beta a_{i,k} -\beta_f  u_{n,i,k} + \mu_{i,k} + \tilde{\delta}_{i,k}.
\end{aligned}
\label{linearized_simple}
\end{equation}
To account for the spacing policy, using $\alpha_f = e^{-\frac{T}{h}}$, the following commonly used first-order filter is introduced \cite{Nunen2019}:
\begin{equation}
    u_{n,i,k+1} = \alpha_f u_{n,i,k} + (1-\alpha_f)\xi_{i,k}.
    \label{ue_filter}
\end{equation}
\subsection{Overlapping dynamics for String Stability Analysis}
Similar to \cite{Silva2025}, combining \eqref{linearized_simple} and \eqref{ue_filter} with the spacing error dynamics  \eqref{distance_error}, and defining the relative velocity between consecutive vehicles, $\Delta v_{i,k}=v_{i-1,k}-v_{i,k}$, the overlapping dynamics $\Sigma_{i,i-1} $ between vehicles $i$ and $i-1$ can be formulated as follows:
\begin{equation}
\left\{
\begin{alignedat}{1}
&\Delta d_{i,k+1}   ={}\Delta d_{i,k} + T\Delta v_{i,k} - hTa_{i,k} 
\\
&\Delta v_{i,k+1}   ={}\Delta v_{i,k} - T a_{i,k} + T a_{i-1,k} 
\\
&a_{i,k+1}          ={}\beta a_{i,k}  -\beta_f u_{n,i,k} + \mu_{i,k}
                         + \tilde{\delta}_{i,k} 
\\
&u_{n,i,k+1}        ={}\alpha_f u_{n,i,k}
                         + (1-\alpha_f)\xi_{i,k} 
\\
&a_{i-1,k+1}        ={}\beta a_{i-1,k} -\beta_f u_{n,i-1,k}
                         + \mu_{i-1,k}+\tilde{\delta}_{i-1,k}
\\
&u_{n,i-1,k+1}      ={}\alpha_f u_{n,i-1,k}
                         + (1-\alpha_f)\xi_{i-1,k}
\end{alignedat}
\right.
\end{equation}
We define 
$x_{1,i,k} =
    \begin{bmatrix}
        \Delta d_{i,k} & \Delta v_{i,k} & a_{i,k} & u_{n,i,k}
    \end{bmatrix}^\top,
$ and $ x_{2,i,k} =
    \begin{bmatrix}
        a_{i-1,k} & u_{n,i-1,k}
    \end{bmatrix}^\top
$. Then, the overlapping dynamics are rewritten as
\begin{align}
    x_{1,i,k+1}  &\!=\! A_1 x_{1,i,k} \!+\! B_1 \xi_{i,k} \!+\! D_1 x_{2,i,k} 
   \! +\! E_1 (\tilde{\delta}_{i,k}\! +\! \mu_{i,k})\notag
    \\
    x_{2,i,k+1} &\!=\! A_2 x_{2,i,k} \!+ \!B_2 \xi_{i-1,k} 
    \!+\! E_2 (\tilde{\delta}_{i-1,k} \!+ \!\mu_{i-1,k})\label{overl}
\end{align}
where
\begin{equation}
\begin{aligned}
    A_1 &= \begin{bmatrix}
      1 & T & -hT & 0 \\
      0 & 1 & -T & 0 \\
      0 & 0 & \beta & 1-\beta \\
      0 & 0 & 0 &  \alpha_f
    \end{bmatrix}, \quad
    B_1 = \begin{bmatrix}
      0 \\ 0 \\ 0 \\ 1- \alpha_f
    \end{bmatrix} \\
    D_1 &= \begin{bmatrix}
      0 & T & 0 & 0 \\
      0 & 0 & 0 & 0 \\
    \end{bmatrix}^\top,
    \quad
    E_1 = \begin{bmatrix}
      0 & 0 & 1 & 0
    \end{bmatrix}^\top
    \\
    A_2 &= \begin{bmatrix}
      \beta & 1-\beta \\
      0 &  \alpha_f
    \end{bmatrix}, \quad
    B_2 = \begin{bmatrix}
      0 \\  1-\alpha_f
    \end{bmatrix}, \quad
    E_2 = \begin{bmatrix}
      1 \\ 0
    \end{bmatrix}
\end{aligned}
\end{equation}
with the particular case $i=1$, the predecessor state corresponds to the leader and satisfies
$
x_{2,1,k+1}=A_2x_{2,1,k}+B_2\xi_{0,k}+E_2\tilde{\delta}_{0,k},
$
with no spacing or relative-velocity error assigned to the leader itself. The following control law is selected for the overlapping system in \eqref{overl}:
\begin{equation}
\xi_{i,k} = K_1 x_{1,i,k} + K_2 x_{2,i,k}
\label{xi_control}
\end{equation}
where $K_1 x_{1,i,k}$ is the feedback component, $K_2 x_{2,i,k}$ is the feedforward component associated with the predecessor states, and $\mu_{i,k}$ is the additional compensation signal. 
The feedback gain $K_1$ is designed to ensure closed-loop stability with disturbance attenuation. The feedforward gain $K_2$ is introduced to counteract the influence of the predecessor input  $u_{n,i-1,k}$, hence, a suitable choice is $K_2=\begin{bmatrix}0 & 1\end{bmatrix}$ as  in \cite{Ploeg2011}, and to preserve the structural properties required for the
string-stability analysis, the feedback gain is restricted to
$
    K_1 =
    \begin{bmatrix}
        k_d & k_v & k_a & 0
    \end{bmatrix}
$.
The residual input $\mu_{i,k}$ is included to attenuate the effects of the disturbance-estimation error $\tilde{\delta}_i(k)$ that remain after DOB compensation.

Here, the residual compensation input $\mu_{i,k}$ is generated using a REN. The parameters used to obtain $\mu_i(k)$ are trained to attenuate the residual effects of the virtual disturbance that are not fully compensated by the DOB and improve the closed-loop performance, while preserving the nominal feedback/feedforward control architecture. The overall closed-loop dynamics are shown in Fig.~\ref{compl}. The string stability is evaluated based on the $\ell_2$ disturbance string stability \cite{Besselink2017}.
\begin{definition}[Disturbance string stability]
Consider a cascade of $N$ systems with states $\bar{x}_i$, indexed by
$i \in \mathcal{N} := \{1,\dots,N\}$, and let $
    \bar{x} := \operatorname{col}(\bar{x}_1,\dots,\bar{x}_N)
$
denote the lumped state. Let $u_0$ denote an exogenous input to the
cascade, 
let $\omega_j$, $j\in\{0,\ldots,N\}$, denote the disturbances,
where $\omega_0$ represents the boundary disturbance associated
with the leader,
and
let
$
    y_i = h(\bar{x}_i)
$
denote a performance output.
The cascade is said to be $\ell_2$ disturbance string stable  under zero initial conditions  if there exist
constants $\gamma_0,\gamma_1 \geq 0$ independent of $i$
such that
\begin{equation}
    \normld{y_i}
    \leq
    \gamma_0 \normld{u_0}
    +
    \gamma_1\sup_{j\in\{0,\ldots,N\}} \normld{\omega_j} 
    \quad
    \forall i \in \mathcal{N}.
\label{definitionstring}
\end{equation}
\end{definition}
 
\begin{figure}
    \centering
    \includegraphics[width=1\linewidth]{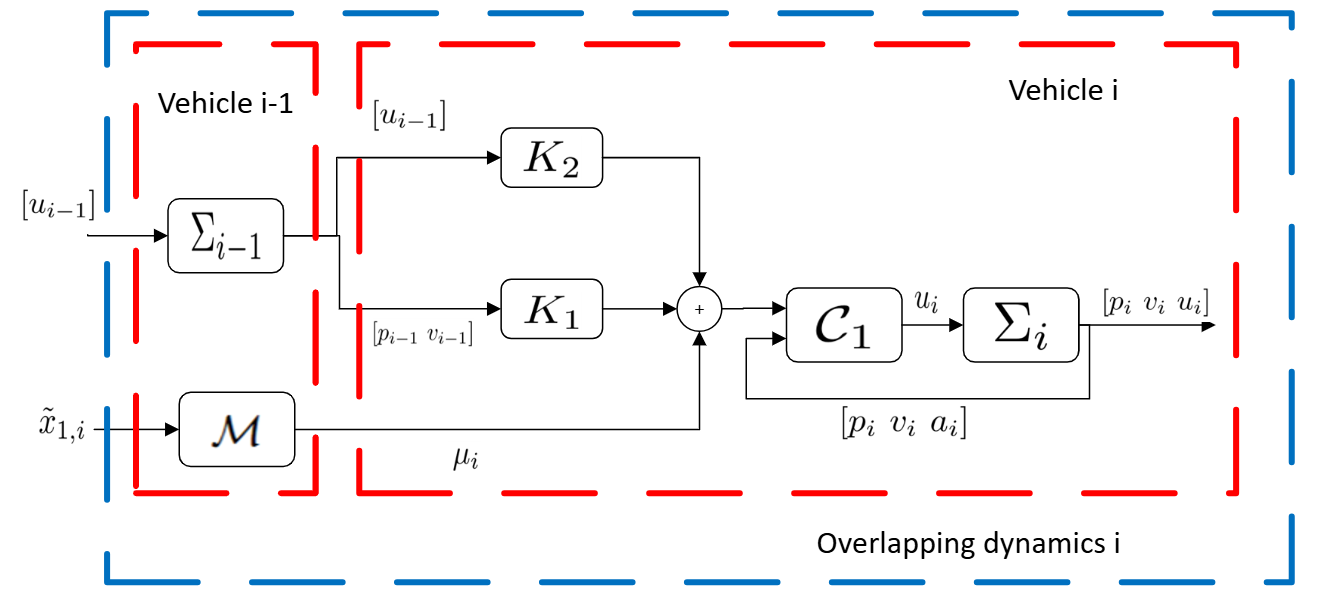}
    \caption{Closed loop of the proposed overlapping dynamics, where $\mathcal{C}_i$ is the feedback linearization controller.}
    \label{compl}
\end{figure}

\section{Controller Design}
This section presents the proposed controller design. We design a nominal feedback controller to ensure individual-vehicle stability and support string stability under bounded perturbations. Then, the REN-based residual compensator is developed to improve the closed-loop performance.
\subsection{Nominal Controller}
The nominal controller is designed to ensure individual-vehicle stability, support string stability and account for perturbations. 
From the overlapping dynamics in \eqref{overl}, the subsystem associated with $x_{1,i}$ can be written as
\begin{align}
    x_{1,i,k+1} &= (A_1+B_1K_1) x_{1,i,k} + (D_1+B_1K_2)x_{2,i,k} ~~~~
    \notag\\
    &+ 
    E_1 \tilde{\delta}_{i,k}
    + E_1\mu_{i,k}
\label{x1subsystem}
\end{align}
\noindent and defining $\mathcal{X}_{i,k} = \begin{bmatrix}
    x_{1,i,k}^\top & x_{2,i,k}^\top
\end{bmatrix}^\top$ 
and
$\omega_{i,k} = \tilde{\delta}_{i,k} + \mu_{i,k}$
the full overlapping system  as
\begin{equation}
\begin{aligned}
\mathcal{X}_{i,k+1}
={}& A_{\mathcal X}\mathcal{X}_{i,k}
+B_{\mathcal X,1}\xi_{i,k}
+B_{\mathcal X,2}\xi_{i-1,k}\\
&+E_{\mathcal X,1}\omega_{i,k}
+E_{\mathcal X,2}\omega_{i-1,k}
\end{aligned}
\label{overlapping_LMI}
\end{equation}
where
$E_{\mathcal X,1}=
\begin{bmatrix}
E_1^\top & 0
\end{bmatrix}^\top,$
$E_{\mathcal X,2}=
\begin{bmatrix}
0 & E_2^\top
\end{bmatrix}^\top,$
and
\begin{equation*}
A_{\mathcal X}=
\begin{bmatrix}
A_1 & D_1\\
0 & A_2
\end{bmatrix}, 
\quad
B_{\mathcal X,1}=
\begin{bmatrix}
B_1 \\ 0
\end{bmatrix},
\quad
B_{\mathcal X,2}=
\begin{bmatrix}
0 \\ B_2
\end{bmatrix}.
\end{equation*}
The performance output is selected as $ y_i = \xi_i$ for the string stability analysis, which is a linear combination of the overlapping states. To establish the frequency condition for disturbance string stability, consider the overlapping dynamics augmented with a one-step delay and $\omega_i = \omega_{i-1}= 0$. Defining $q_{i,k+1}=\xi_{i-1,k}$, and the augmented state $\mathcal X_{q,i,k}=[\mathcal X_{i,k}^{\top}\;q_{i,k}]^{\top}$. Separating the feedback and feedforward control components as $\xi_{i,k}=K_1x_{1,i,k}+F\mathcal X_{i,k}$, where $F = \begin{bmatrix}
    0_{1\times4} & K_2
\end{bmatrix}$ yields the augmented system 
\begin{equation}
    \begin{aligned}
\mathcal X_{q,i,k+1}&=A_QX_{q,i,k}+B_{Q,1}K_1x_{1,i,k}+B_{Q,2}\xi_{i-1,k},\\
y_{Q,i,k}&=C_QX_{q,i,k}+D_{Q,1}K_1x_{1,i,k}+D_{Q,2}\xi_{i-1,k},
\label{over_augmented}
\end{aligned}
\end{equation}
with performance output $
y_{Q,i,k}
=
\begin{bmatrix}
\xi_{i,k} &
\nu(\xi_{i-1,k}-q_{i,k})
\end{bmatrix}^{\top}
$, 
scalar $\nu > 0$
and matrices
$
B_{Q,1}=
\begin{bmatrix}
B_{\mathcal X,1}^{\top}&0
\end{bmatrix}^{\top}
$,
$
B_{Q,2}=
\begin{bmatrix}
B_{\mathcal X,2}^{\top}&1
\end{bmatrix}^{\top},
$
$D_{Q,1}=
\begin{bmatrix}
1&0
\end{bmatrix}^{\top},
$
$
D_{Q,2}=
\begin{bmatrix}
0&\nu
\end{bmatrix}^{\top}
$
and
\begin{equation*}
A_Q =
\begin{bmatrix}
A_{\mathcal X}+B_{\mathcal X,1}F & 0\\
0 & 0
\end{bmatrix},
\quad
C_Q=
\begin{bmatrix}
F&0\\
0_{1\times6}&-\nu
\end{bmatrix}.
\end{equation*}
Let $\mathcal T(z)$ denote the nominal transfer
function from $\xi_{i-1}$ to $\xi_i$ in \eqref{over_augmented}, then next proposition follows.
\begin{proposition}
 Consider the overlapping system \eqref{overlapping_LMI} and the augmented system \eqref{over_augmented}. For positive $\gamma_d>0$, 
suppose there exist
$Q_c\in\mathbb{S}_{+}^{3}$,
$Q_f\in\mathbb{S}_{+}^{4}$,
$q_u>0$, and
$X_c\in\mathbb{R}^{1\times3}$,
such that the following LMI conditions are jointly satisfied:
\begin{align}
&\begin{bmatrix}
-Q_1 & * & * & *\\
0 & -\gamma_d^2 I & * & *\\
A_1Q_1+B_1X_1 & E_1 & -Q_1 & *\\
Q_1 & 0 & 0 & -I
\end{bmatrix}
\prec 0
\label{lmiindividual}\\
&\begin{bmatrix}
-Q_2 & * & * & *\\
0 & -I & * & *\\
A_QQ_2+B_{Q,1}X_2 & B_{Q,2} & -Q_2 & *\\
C_QQ_2+D_{Q,1}X_2 & D_{Q,2} & 0 & -I
\end{bmatrix}
\preceq 0
\label{lmistring}
\end{align}
where
$Q_1 = \operatorname{blkdiag}(Q_c,q_u),$
$Q_2 = \operatorname{blkdiag}(Q_c,Q_f),$
$X_1 = [\,X_c\;\;0\,],$
and
$X_2 = \begin{bmatrix}
    X_c & 0_{1\times4}
\end{bmatrix}.$
Then, the feedback gain
$
K_1=X_1Q_1^{-1}
$
guarantees that the nominal local closed-loop dynamics
are Schur stable, with an induced
$\ell_2$ gain bounded by $\gamma_d$ from the local disturbance to $x_{1,i}$,
and that
the transfer function
$\mathcal{T}$ satisfies
$
|\mathcal{T}(e^{\mathrm{j}\theta})|^2
+\nu^2|1-e^{-\mathrm{j}\theta}|^2
\leq 1.
$
\end{proposition}

\begin{proof}
Let $P_j = Q^{-1}_j$, $V_1=x_{1,i,k}^{\top}P_1x_{1,i,k}$ and $
V_2=\mathcal X_{q,i,k}^{\top}
P_2\mathcal X_{q,i,k}$. Applying  Schur complement to \eqref{lmiindividual} and \eqref{lmistring} yields the inequalities
$
V_1(k+1)-V_1(k)
+\normtwo{x_{1,i,k}}^2
-\gamma_d^2|\omega_{i,k}|^2<0
$ and 
 
$V_2(k+1)-V_2(k)
+|\xi_{i,k}|^2
+\nu^2|\xi_{i-1,k}-q_{i,k}|^2
-|\xi_{i-1,k}|^2
\leq 0$, respectively.
\end{proof}
 \color{black}
 
\color{black}
\subsection{Neural Compensation}
To design the neural compensation, consider the stage cost
$
    \mathcal{J}_{i,k}
    =
    Q_x(x_{1,i,k})
    +
    q_d \tilde{\delta}_{i,k}^2
    +
    r_\mu \mu_{i,k}^2,
$
where $Q_x(\cdot)$ is a positive definite function penalizing the vehicle
state, and $q_d>0$ and $r_\mu>0$ weight the disturbance-estimation
error and the neural compensation effort. The
infinite-horizon cost is
$
J_i
    =
    \sum_{k=0}^{\infty}\mathcal{J}_{i,k}.
$
Since $\tilde{\delta}_{i,k}$ is not directly available, a one-step
nominal predictor is introduced as
\begin{equation}
    \hat{x}_{1,i,k+1}
    =
    A_1x_{1,i,k}
    +
    B_1\xi_{i,k}
    +
    D_1x_{2,i,k}
    +
    E_1\mu_{i,k}.
    \label{nominal_predictor}
\end{equation}
Using \eqref{nominal_predictor} and \eqref{overl}, we can derive the prediction error
$
    \tilde{x}_{1,i,k+1}
    =
    x_{1,i,k+1}
    -
    \hat{x}_{1,i,k+1}
    =
    E_1\tilde{\delta}_{i,k}.
$
 
The objective is to determine a neural compensation policy
$\mu_{i,k}=\pi(s_{i,k})$ that minimizes the cost function
 
,
$
   \pi^*
    =
    \arg\min_{\pi} J_i.
$
  
For a fixed policy $\pi$, the following action-value function satisfies the Bellman relation:
\begin{equation}
    Q^\pi(s_{i,k},\mathcal{X}_{i,k},\mu_{i,k})
    =
    \mathcal{J}_{i,k}
    +
    Q^\pi(s_{i,k+1},\mathcal{X}_{i,k+1},\mu_{i,k+1})
    \label{qfunction}
\end{equation}
where $s_{i,k}$ denotes the information supplied to the neural policy
and $\mu_{i,k+1}=\pi(s_{i,k+1})$.
The action-value function is approximated using the parameterized
structure
\begin{equation}
\begin{aligned}
\hat Q^\pi_{i,k}
={}&
\mathcal{J}_{i,k}
+\hat{\theta}\mu_{i,k}^2
+\mu_{i,k}y_{r,i,k}
+\mathcal{X}_{i,k}^{\top}\hat S \mathcal{X}_{i,k}
\end{aligned}
\label{Qapprox}
\end{equation}
where $\hat\theta\geq\bar{\theta}>0$, 
$\hat{S}\in\mathbb{S}^{6}$, and $y_{r,i,k}$ is generated by an
acyclic REN \cite{Revay2024} given by
\begin{equation}
    \begin{bmatrix}
        \chi_{i,k+1} \\
        \psi_{i,k} \\
        y_{r,i,k}
    \end{bmatrix}
    =
    \underbrace{
    \begin{bmatrix}
        \mathcal{A}_{n} & \mathcal{B}_{1} & \mathcal{B}_{2} \\
        0 & \mathcal{D}_{11} & \mathcal{D}_{12} \\
        \mathcal{C}_{2} & \mathcal{D}_{21} & \mathcal{D}_{22}
    \end{bmatrix}
    }_{\hat{W}}
    \begin{bmatrix}
        \chi_{i,k} \\
        \varphi_{i,k} \\
        s_{i,k}
    \end{bmatrix}
\label{REN}
\end{equation}
with $\varphi_{i,k}=\sigma(\psi_{i,k}),$
where $\sigma(\cdot)$ is an element-wise activation function,
$\mathcal{D}_{11}\in\mathbb{R}^{n_d\times n_d}$ is strictly lower
triangular, $\chi_{i,k}\in\mathbb{R}^{n_q}$ denotes the REN internal
state, and $s_{i,k}\in\mathbb{R}^{n_s}$ denotes the REN input. For this application, $s_{i,k} =\begin{bmatrix}
    x_{1,i,k}^\top & \tilde{x}_{1,i,k}^\top
\end{bmatrix}^\top$ is selected.
For fixed $\pi$, the Q-function approximator is trained offline using stored trajectories. Define $\tilde Q_j^\pi=\hat Q_j^\pi-\bar Q_j^\pi$, where $\bar Q_j^\pi$ is a finite-horizon rollout target. The parameters of \eqref{Qapprox} are trained to minimize 
$
\mathcal{L}_{\pi}
=
\frac{1}{N_{\mathrm{tr}}}
\sum_{j=1}^{N_{\mathrm{tr}}}
\left(\tilde{Q}_{j}^{\pi}\right)^2
$, where $N_{\mathrm{tr}}$ denotes the number of stored training samples.
To preserve stability properties, the trained REN parameters are projected
onto the set satisfying the following LMI condition: 
\begin{equation}
\begin{bmatrix}
\bar\Xi & \bar G^\top & \bar J^\top\\
\bar G & \mathcal Q & 0\\
\bar J & 0 & \gamma_r I
\end{bmatrix}
\succ 0
\label{Q_LMI}
\end{equation}
with
\[
\begin{aligned}
\bar\Xi &=
\begin{bmatrix}
\mathcal Q & 0 & 0\\
0 & 2I-\mathcal D_{11}-\mathcal D_{11}^{\top}
  & -\mathcal D_{12}\\
0 & -\mathcal D_{12}^{\top} & \gamma_r I
\end{bmatrix}\\
    \bar G &=
\begin{bmatrix}
\mathcal A_n\mathcal Q & \mathcal B_1 & \mathcal B_2
\end{bmatrix},
\quad
\bar J =
\begin{bmatrix}
\mathcal C_2\mathcal Q & \mathcal D_{21} & \mathcal D_{22}
\end{bmatrix}.
\end{aligned}
\]
Note that condition \eqref{Q_LMI} ensures an input-output gain $\gamma_r$ for the REN \cite{Revay2024}. With fixed $\mathcal A_n$ and $\mathcal C_2$ as scalar design parameters, the trained REN parameters are projected onto the admissible set by solving
\begin{equation}
\begin{aligned}
\min_{\substack{
    Q,\mathcal B_i, \mathcal D_{ij}
}}
\quad &
\sum_{i=1}^{2}
    \left\| \mathcal B_i-\hat{\mathcal B}_i \right\|_F^2
+
\sum_{i=1}^{2}\sum_{j=1}^{2}
    \left\| \mathcal D_{ij}-\hat{\mathcal D}_{ij} \right\|_F^2
\\
\text{s.t.}\quad & \eqref{Q_LMI}.
\end{aligned}
\end{equation}
The scalar parameter is projected independently as
$
\hat\theta
\leftarrow
\max\{\hat\theta,\bar\theta\}.
$
In consequence, the induced gain from the REN input to the compensation
signal is bounded by
$
\gamma_m
=
\frac{\gamma_r}
{2(r_\mu+\bar{\theta})}.
$The bound $\gamma_r$ is selected such that $\gamma_m$ satisfies the
small-gain conditions established in Section~IV.
After policy
evaluation step is completed, the policy is improved by selecting
the action that minimizes the learned action-value function,
$
    \mu_{i,k}^{\mathrm{new}}
    =
    \arg\min_{\mu_{i,k}}
    \hat Q^\pi_{i,k}.
$
which can be computed explicitly from \eqref{Qapprox} as 

\begin{equation}
    \mu_{i,k}^{\mathrm{new}}
    =
    -\frac{
    y_{r,i,k}}
    {2(r_\mu+\hat{\theta})}.
\end{equation}
The policy-evaluation and policy-improvement steps are repeated
according to the policy-iteration procedure.

\section{Main Results}
This sections presents the stability analysis of individual-vehicle and disturbance-string stability. 

\begin{theorem}
Consider subsystem \eqref{x1subsystem} with zero
predecessor excitation, $x_{2,i}=0$. Suppose that
\eqref{lmiindividual} is feasible with gain
$\gamma_d$, and that the REN satisfies \eqref{Q_LMI}, yielding the
compensation-policy gain $\gamma_m$. If
$
\gamma_d\gamma_m<1,
$
then the local feedback interconnection between the nominal vehicle
dynamics and the neural compensation is $\ell_2$ stable. Moreover,
the induced gain from $\tilde\delta_i$ to $x_{1,i}$ is finite.
\end{theorem}

\begin{proof}
From \eqref{lmiindividual} and \eqref{Q_LMI}, we obtain 

\begin{equation}
    \|x_{1,i}\|_{\ell_2}
    \leq
    \gamma_d
    \left(
        \|\mu_i\|_{\ell_2}
        +
        \|\tilde{\delta}_i\|_{\ell_2}
    \right)
    \label{p1}
\end{equation}
\begin{equation}
   \normld{\mu_i}
    \leq
    \gamma_m
    \left( \normld{x_{1,i}}
        +
        \normld{\tilde{x}_{1,i}}
    \right)
\end{equation}
Since
$
\tilde{x}_{1,i,k+1}
=
E_1\tilde{\delta}_{i,k},
$, then
$
    \normld{\tilde{x}_{1,i}}
    =
    \normld{\tilde \delta_i}
$
and
\begin{equation}
    \|\mu_i\|_{\ell_2}
    \leq
    \gamma_m
    \left(
        \|x_{1,i}\|_{\ell_2}
        +
        \|\tilde{\delta}_i\|_{\ell_2}
    \right).
    \label{p3}
\end{equation}
Substituting \eqref{p3} into \eqref{p1}, it follows that 
\begin{equation}
    \|x_{1,i}\|_{\ell_2}
    \leq
    \frac{\gamma_d(1+\gamma_m)}
         {1-\gamma_d\gamma_m}
    \|\tilde{\delta}_i\|_{\ell_2}.
    \label{fineq}
\end{equation}
Therefore, the induced $\ell_2$ gain from $\tilde{\delta}_i$ to
$x_{1,i}$ is finite. Furthermore, in the absence of external
disturbances, the small-gain condition
$\gamma_d\gamma_m<1$ guarantees internal stability of the feedback
interconnection.
\end{proof}

 \begin{theorem}
Consider the platoon under the control law \eqref{xi_control} and zero
initial conditions. Suppose that \eqref{lmiindividual} and
\eqref{lmistring} are feasible, $\xi_0\in\ell_2$, and $
\bar\omega
=
\sup_{j\in\{0,\ldots,N\}}
\|\omega_j\|_{\ell_2}<\infty,
$
then with the performance output $y_i=\xi_i$, the closed-loop platoon
is $\ell_2$ disturbance string stable.
\end{theorem}

\begin{proof}
Under zero initial conditions, taking the $z$-transform of the
closed-loop dynamics yields
 
\begin{equation}
\Xi_i(z)
=
\mathcal{T}(z)\Xi_{i-1}(z)
+
P_c(z)\Omega_i(z)
+
P_p(z)\Omega_{i-1}(z)
\label{eq:local_transfer}
\end{equation}
where $\mathcal{T}(z)$ denotes the closed-loop transfer function from
$\xi_{i-1}$ to $\xi_i$, while $P_c(z)$ and $P_p(z)$ denote the
transfers from $\omega_i$ and $\omega_{i-1}$ to $\xi_i$, respectively.

Define $\Delta_z=z-1$,
$N_a=k_a\Delta_z^2-T(k_v+hk_d)\Delta_z-k_dT^2$,
$
N_p=T(k_dT+k_v\Delta_z)$
and 
$
D_c=\Delta_z^2(z-\alpha_f)(z-\beta)
-(1-\alpha_f)(1-\beta)N_a
$. Then, we have 
\begin{equation*}
\begin{aligned}
\mathcal T(z)
&=\frac{(1-\alpha_f)
[\Delta_z^2(z-\beta)+(1-\beta)N_p]}{D_c}\\
P_c(z)
&=\frac{(z-\alpha_f)N_a}{D_c},
\quad
P_p(z)=\frac{(z-\alpha_f)N_p}{D_c}.
\end{aligned}
\end{equation*}
Define $
R(z)=Tk_d[T-(1-\alpha_f)h]
+\Delta_z[(1-\alpha_f)k_a+Tk_v]
$ and $P_f(z)=\mathcal T(z) P_c(z)+P_p(z)$. Then, it follows that 
\begin{equation}
P_f(z)
=\Delta_z^3\frac{(z-\alpha_f)(z-\beta)R(z)}{D_c^2}
=(z-1)^3\bar P_f(z).
\end{equation}
Since $A_1+B_1K_1$ is Schur and
$D_c(z)=\det(zI-A_1-B_1K_1)$, it follows that
$\bar P_f\in\mathcal{H}_\infty$.
 
Recursive substitution of \eqref{eq:local_transfer} yields
\begin{align}
\Xi_i(z)
={}&
\mathcal{T}(z)^i\Xi_0(z)
+
P_c(z)\Omega_i(z)
+
\mathcal{T}(z)^{i-1}P_p(z)\Omega_0(z) \notag
\\
&+
\sum_{j=1}^{i-1}
\mathcal{T}(z)^{i-j-1}
P_f(z)\Omega_j(z).
\label{eq:recursive_string}
\end{align}
Since the nominal closed-loop dynamics are Schur stable,
$\mathcal{T}(z)$, $P_c(z)$, and $P_p(z)$ are stable transfer
functions. Defining
$
p_c=\normH{P_c},
$
$p_p=\normH{P_p},
$
and using
$\normH{\mathcal T}\leq1
$
from \eqref{lmistring}, it follows that 
\begin{equation}
\normld{\xi_i}
\leq
\normld{\xi_0}
+
\left(
p_c+p_p+
\sum_{j=1}^{i-1}
\normH{\mathcal{T}^{i-j-1}P_f}
\right)\bar\omega.
\label{eq:dss_prebound}
\end{equation}
From \eqref{lmistring},
$
|\mathcal{T}(e^{\mathrm{j}\theta})|^2
+\nu^2|1-e^{-\mathrm{j}\theta}|^2
\leq1.
$
Define $r_z=|1-e^{-\mathrm{j}\theta}|$. Since
$0\leq\nu^2r_z^2\leq1$, using
$\sqrt{1-x}\leq e^{-x/2}$ gives
$
|\mathcal{T}(e^{\mathrm{j}\theta})|^{i-j-1}
\leq
e^{-(i-j-1)\nu^2r_z^2/2}.
$
From
$
P_f(z)=(z-1)^3\bar P_f(z),
$
defining
$
\eta=\|\bar P_f\|_{\mathcal H_\infty}<\infty
$
gives
$
|P_f(e^{\mathrm{j}\theta})|\leq\eta r_z^3.
$
Thus, for $j=1,\ldots,i-2$,
\begin{equation}
\|\mathcal{T}^{i-j-1}P_f\|_{\mathcal H_\infty}
\leq
\eta
\left(\frac{3}{e\nu^2}\right)^{3/2}
(i-j-1)^{-3/2}.
\end{equation}
Denote $p_f=\|P_f\|_{\mathcal H_\infty}$, it follows that 
\begin{equation}
\sum_{j=1}^{i-1}
\|\mathcal{T}^{i-j-1}P_f\|_{\mathcal H_\infty}
\leq
p_f+
\eta
\left(\frac{3}{e\nu^2}\right)^{3/2}
\zeta\!\left(\frac32\right)
\end{equation}
with $\zeta\!\left(\frac32\right)$ being the Riemman zeta function evaluated at $\frac{3}{2}$. Hence, it follows from \eqref{eq:dss_prebound} that 
\begin{equation}
\|\xi_i\|_{\ell_2}
\leq
\|\xi_0\|_{\ell_2}
+
\gamma_1\bar\omega
\label{eq:xi_omega_bound}
\end{equation}
where
$
\gamma_1=
p_c+p_p+p_f+
\eta
\left(\frac{3}{e\nu^2}\right)^{3/2}
\zeta\!\left(\frac32\right).
$
Since $\gamma_1<\infty$ is independent of $i$, condition
\eqref{eq:xi_omega_bound} implies \eqref{definitionstring}, with $y_i=\xi_i$, $u_0=\xi_0$, and
$\gamma_0=1$.
\end{proof}

\noindent\textit{Corollary:}
Let $T_x$, $P_{x,c}$, and $P_{x,p}$ denote the stable transfer
functions from $\xi_{i-1}$, $\omega_i$, and $\omega_{i-1}$ to
$x_{1,i}$, respectively, and define
$
g_\xi=\|T_x\|_{\mathcal H_\infty},
$
$
g_\omega=
\normH{P_{x,c}}
+
\|P_{x,p}\|_{\mathcal H_\infty}.
$
If
\begin{equation}
\gamma_m(g_\xi\gamma_1+g_\omega)<1
\label{eq:string_small_gain}
\end{equation}
and
$
\bar\delta=
\sup_j\|\tilde\delta_j\|_{\ell_2}<\infty,
$
then the platoon is $\ell_2$ disturbance string stable with respect to $\tilde\delta_i$

\begin{proof}    
Taking the z-transform for the overlapping dynamics \eqref{overl} and defining $A_c = A_1+B_1K_1$, we have 
\begin{align}
X_{1,i}(z)
&=(zI-A_c)^{-1}
[(D_1+B_1K_2)X_{2,i}(z)
+E_1\Omega_i(z)]\notag\\
X_{2,i}(z)
&=(zI-A_2)^{-1}
[B_2\Xi_{i-1}(z)
+E_2\Omega_{i-1}(z)].
\label{zoverlap}
\end{align}
Substituting $X_{2,i}(z)$ into $X_{1,i}$ in \eqref{zoverlap}, it follows that 
\begin{equation}
X_{1,i}(z)
=
T_x\Xi_{i-1}(z)
+
P_{x,c}\Omega_i(z)
+
P_{x,p}\Omega_{i-1}(z).
\label{tfX}
\end{equation}
Taking  the induced $\ell_2$ gain from \eqref{tfX}, we can derive   
\begin{equation}
    \normld{x_{1,i}} \leq 
    g_\xi\normld{\xi_{i-1}} +g_\omega\bar\omega.
    \label{intdd}
\end{equation}
Then, combining \eqref{intdd} and \eqref{eq:xi_omega_bound}, it follows that 
\begin{equation}
\|x_{1,i}\|_{\ell_2}
\leq
g_\xi\|\xi_0\|_{\ell_2}
+
(g_\xi\gamma_1+g_\omega)\bar\omega.
\label{eq:x_string_bound}
\end{equation}
Since $\normld{\omega_i} \leq \normld{\tilde\delta_i}+\normld{\mu_i}$, using \eqref{p3} and \eqref{eq:x_string_bound} gives 
\begin{equation}
\bar\omega
\leq
(1+\gamma_m)\bar\delta
+
\gamma_m g_\xi\|\xi_0\|_{\ell_2}
+
\gamma_m(g_\xi\gamma_1+g_\omega)\bar\omega.
\label{eq:x_string_bound11}
\end{equation}
Thus, under \eqref{eq:string_small_gain}, it follows from \eqref{eq:x_string_bound11} that 
\begin{equation}
\bar\omega
\leq
\frac{
\gamma_mg_\xi\|\xi_0\|_{\ell_2}
+
(1+\gamma_m)\bar\delta
}{
1-\gamma_m(g_\xi\gamma_1+g_\omega)
}
<\infty.\label{eqlalala}
\end{equation}
Substituting \eqref{eqlalala} into \eqref{eq:xi_omega_bound}, we can establish the
disturbance string stability with respect to $\tilde\delta_i$.
\end{proof}

\section{Experimental Results}

\begin{figure}[t]
    \centering
    \includegraphics[width=0.40\linewidth]{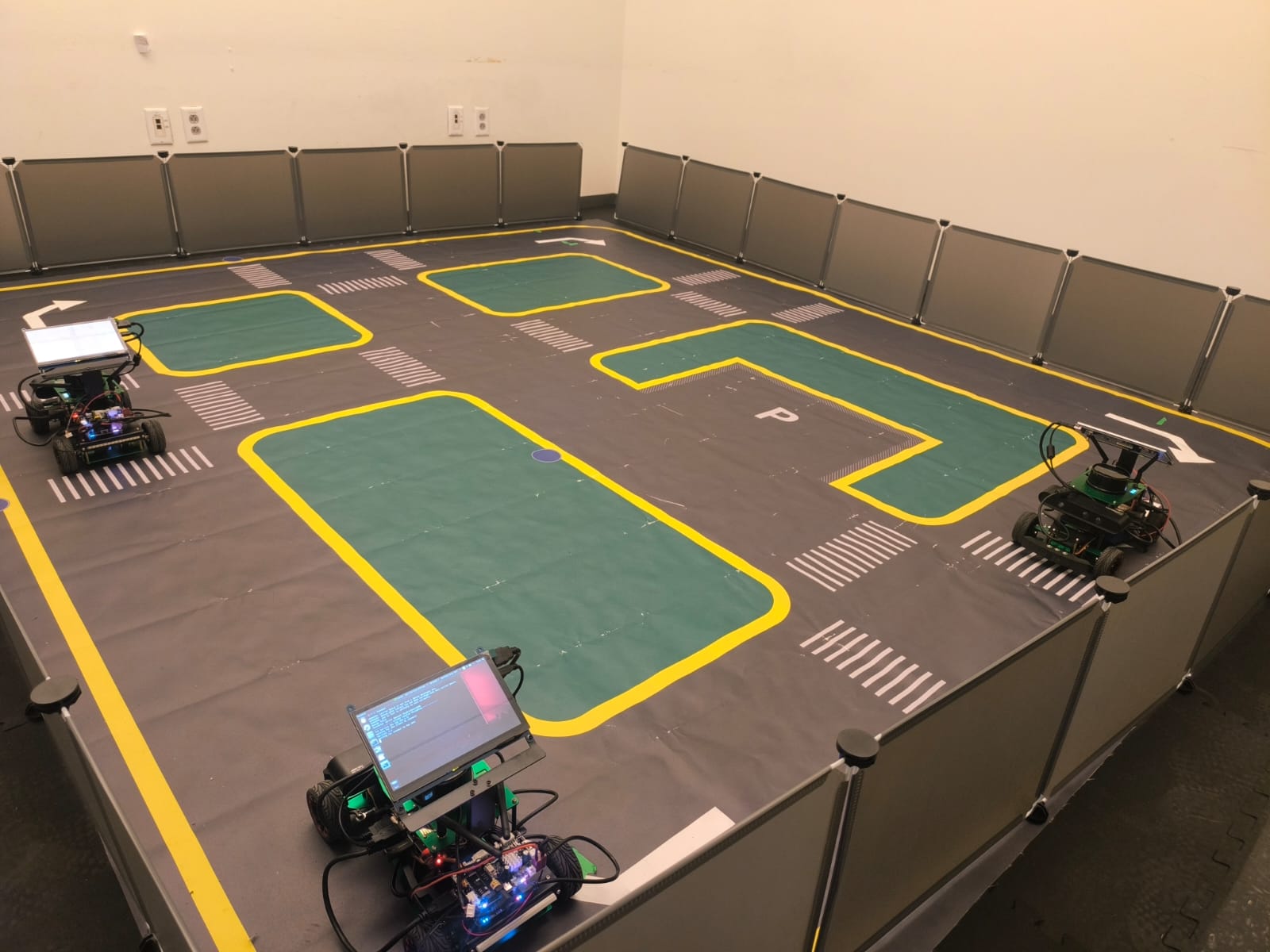}
    \includegraphics[width=0.55\linewidth]{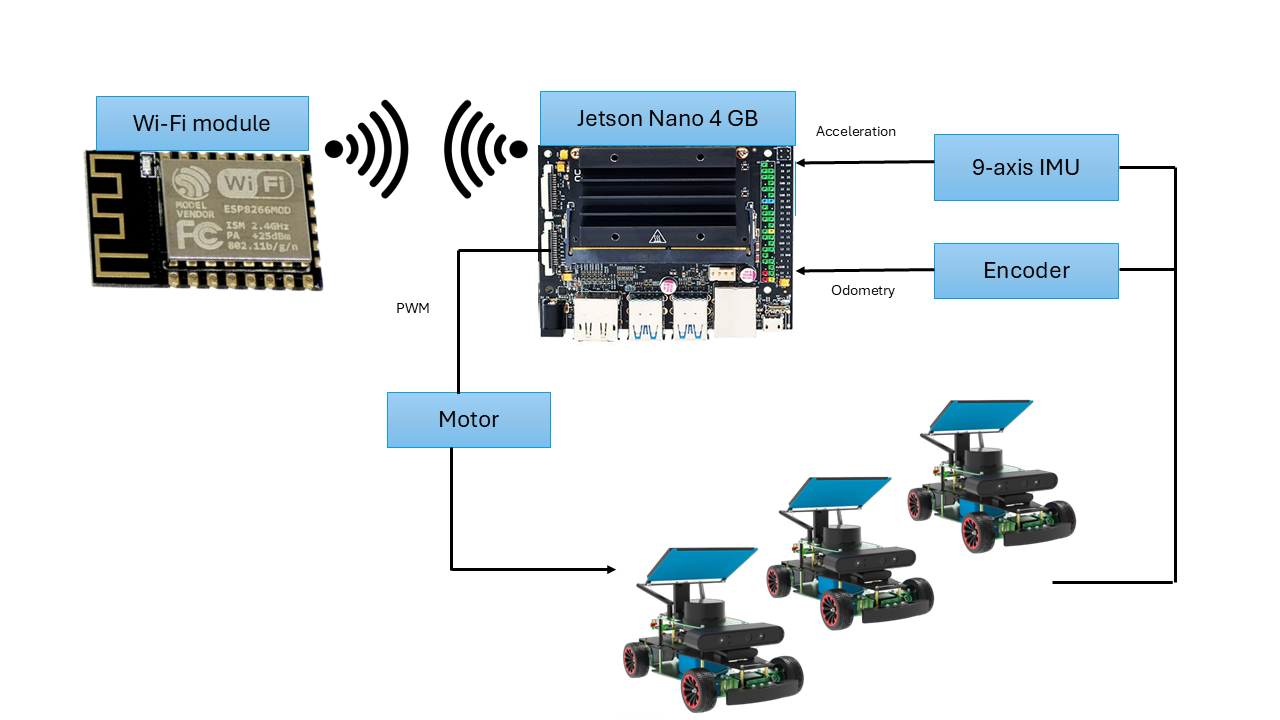}
    \caption{Experimental test setup and onboard control architecture.}
    \label{expsetup}
\end{figure}
Physical experiments were conducted using three vehicles Ackermann-steered vehicles, consisting of one leader and two followers, as shown in Fig. \ref{expsetup}, traveling along the track's outer path.Lateral path following was handled independently, as this study focuses only on longitudinal platoon control. 
Each vehicle is equipped with an onboard
Jetson Nano 4 GB computer and exchanges predecessor information through
Wi-Fi. The control algorithms are executed locally at a sampling frequency of 50 Hz. The nominal parameters used for controller design are given in Table \ref{parameters}. 
\begin{table}[!ht]
    \caption{Nominal parameters used for controller design.}
    \centering
    \label{parameters}
    \begin{tabular}{c c c c }
        \hline \hline
        Parameter & Value & Parameter & Value \\
        \hline
        $m$ & 4 kg & $\tau$ & 0.64 s \\
        $T$ & 0.02 s & c  & 0 kg/m \\
        $h$ & 1 s & $r$ & 1 m \\
        $\beta$ & 0.1 & $l_d$ &  0.02      \\ 
        \hline \hline
    \end{tabular}
\end{table}
Both controllers were evaluated using a predefined leader velocity profile for velocity tracking and spacing regulation. 
The gains used for the nominal controller were $K_1 = \begin{bmatrix}
0.735 & 1.596 & -1.605 & 0
\end{bmatrix}$ and  $K_2 = \begin{bmatrix}
0 & 1
\end{bmatrix}$. The REN employed had one recurrent state $n_q=1$ and nonlinear dimension $n_d=2$. 
Fig. \ref{ex1} shows the experimental results using only the nominal controller, while Fig. \ref{rex2} uses the proposed controller. Using only the nominal control, both followers track the leader velocity profile, but residual model mismatch still produces noticeable oscillations in the velocity and spacing errors. Our proposed controller reduces these errors as shown in Fig. \ref{rex2}. To quantify the performance, the root mean square error (RMSE) was calculated as $\mathrm{RMSE}
=
\sqrt{\frac{1}{n}
\sum_{k=1}^{n}
(\Delta e_{k})^2}$, with $e_k$ representing either the relative speed or spacing error and $n$ the number of measurements. Table \ref{performance} summarizes the RMSE values obtained for both controllers and both followers
The REN-enhanced
controller reduces the relative-velocity RMSE
by approximately 32.6\% and 29.6\% for
followers 1 and 2, respectively, while the spacing error RMSE reductions are 40.9\% and 12.2\%. These results
demonstrate the experimental performance
improvement achieved by the learned residual
compensation.

\begin{figure}[!ht]
    \centering
    \includegraphics[width=0.49\linewidth]{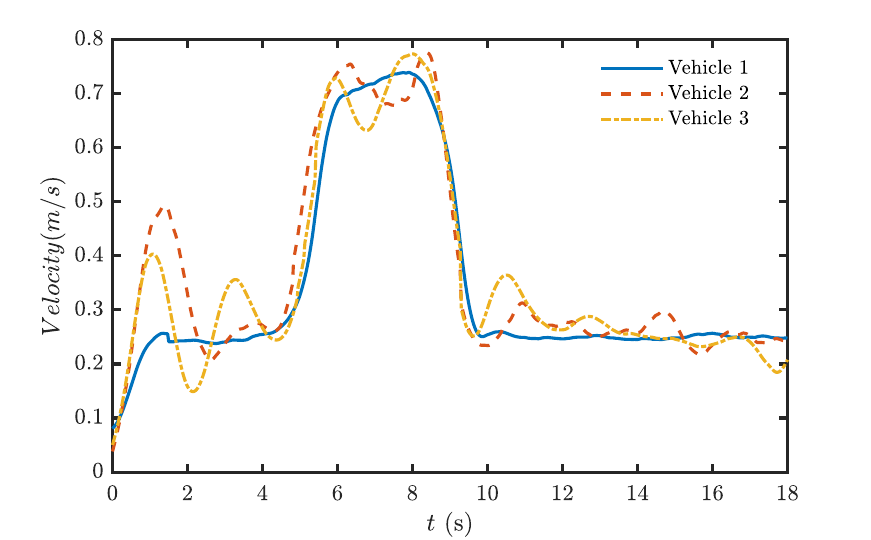}
    \includegraphics[width=0.49\linewidth]{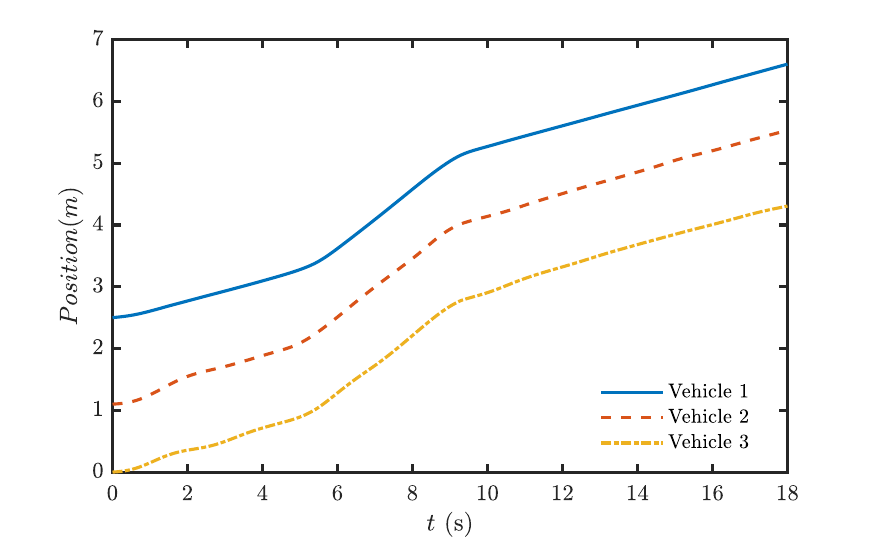}
    \centering
    \includegraphics[width=0.49\linewidth]{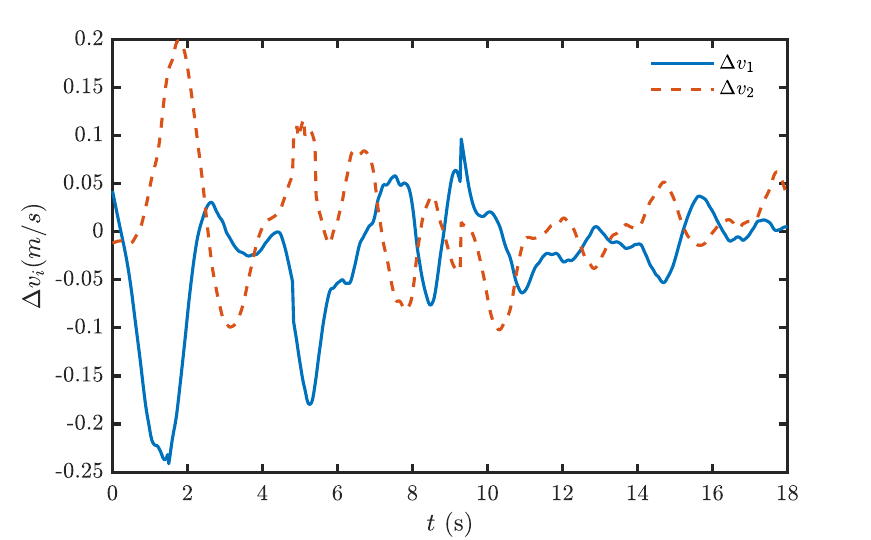}
    \includegraphics[width=0.49\linewidth]{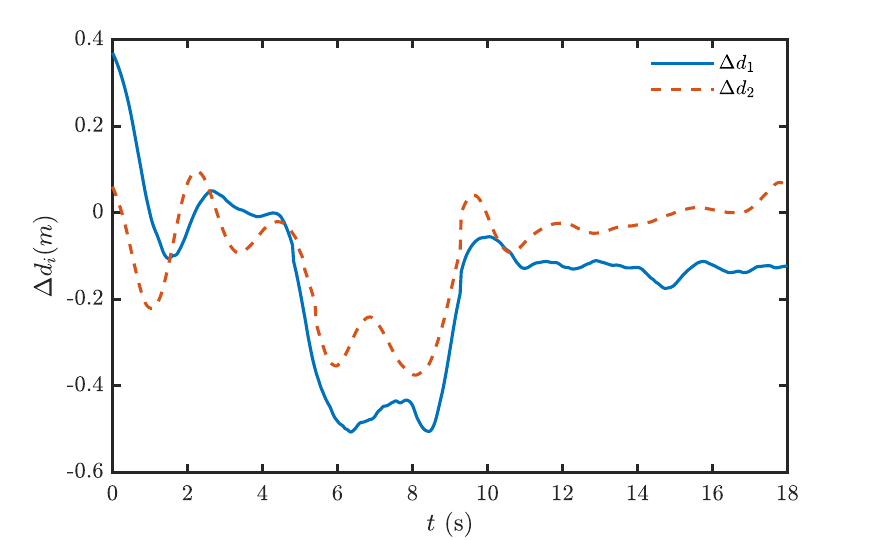}
    \caption{Velocity and spacing errors using the nominal controller.}
    \label{ex1}
\end{figure}

\begin{figure}[t]
    \centering
    \includegraphics[width=0.49\linewidth]{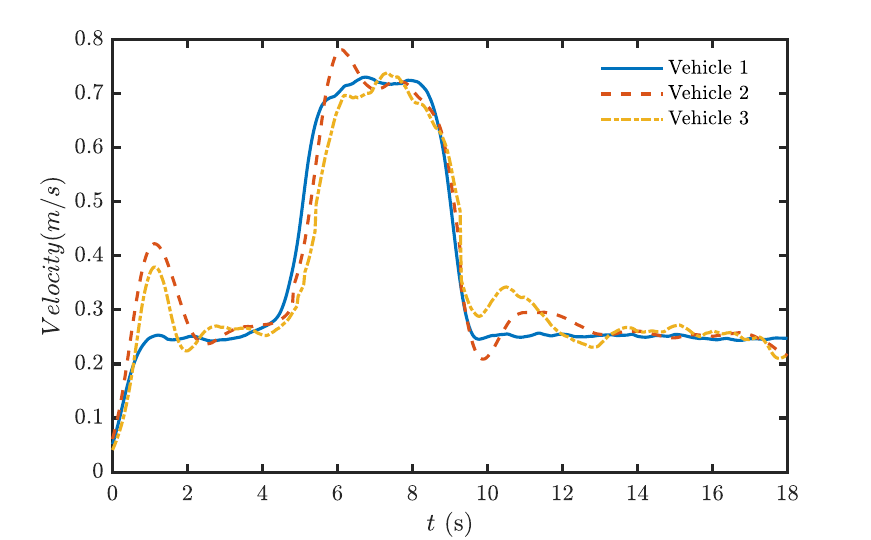}
    \includegraphics[width=0.49\linewidth]{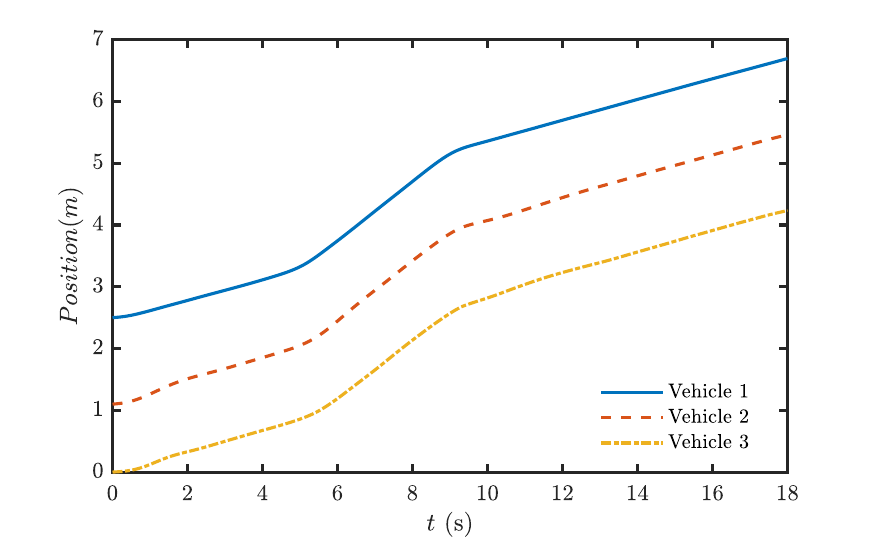}
    \centering
    \includegraphics[width=0.49\linewidth]{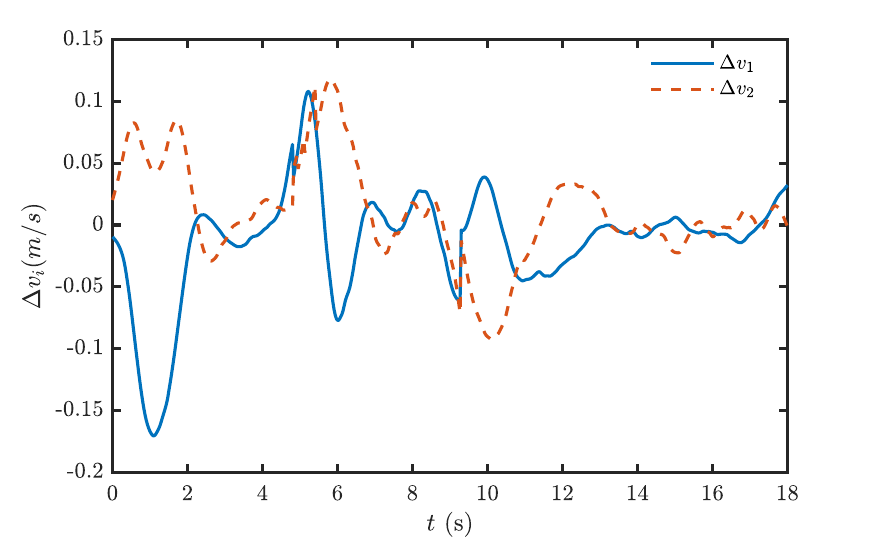}
    \includegraphics[width=0.49\linewidth]{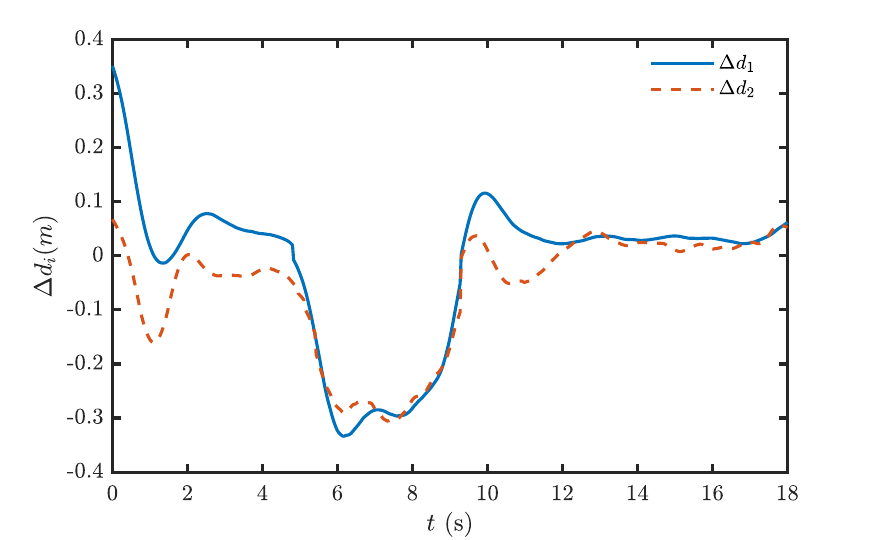}
    \caption{Velocity and spacing errors using the proposed scheme.}
    \label{rex2}
\end{figure}

\begin{table}[t]
\centering
\caption{Experimental performance comparison.}
\label{performance}
\small
\begin{tabular}{lccc}
\hline \hline
Performance indicator & Nominal & Proposed  \\
\hline
Velocity RMSE follower 1 (m/s) & 0.063042  & 0.042500  \\
Velocity RMSE follower 2 (m/s) & 0.056391  & 0.039722  \\
Spacing RMSE follower 1 (m) & 0.216918  & 0.128285  \\
Spacing RMSE follower 2 (m) & 0.135896  & 0.119276  \\
\hline \hline
\end{tabular}
\end{table}
 
\section{Conclusion}
This paper developed a residual learning-based control framework for vehicle platoons by combining nominal model-based control, disturbance observer compensation, and a REN-based residual policy. The nominal controller provides local disturbance attenuation and supports disturbance string stability, while the REN is trained offline to compensate for residual model mismatch. A gain-constrained REN projection, together with small-gain conditions, preserves the $\ell_2$ stability of the local feedback interconnection and the platoon. Experimental results show that the proposed compensation reduces spacing and velocity errors compared with the nominal controller alone, illustrating the performance benefits of the residual learning component. Future work will focus on relaxing the sufficient gain conditions to allow more flexible controller designs.

\bibliographystyle{IEEEtran}
\bibliography{bibliography}

@article{Silva2025,
   author = {Rafael Silva and Anh-Tu Nguyen and Thierry-Marie Guerra and Fernando Souza and Luciano Frezzatto},
   number = {4},
   journal = {IEEE Trans. Intell. Veh},
   pages = {2876-2891},
   title = {Switched Dynamic Event-Triggered Control for String Stability of Nonhomogeneous Vehicle Platoons With Uncertainty Compensation},
   volume = {10},
   year = {2025}
}

@article{Ploeg2014,
   author = {Jeroen Ploeg and Nathan Wouw and Henk Nijmeijer},
   number = {2},
   journal = {IEEE Trans. Control Syst. Technol.},
   month = mar,
   pages = {786-793},
   title = {Lp string stability of cascaded systems: Application to vehicle platooning},
   volume = {22},
   year = {2014}
}

@inproceedings{Min2023,
   author = {Youngjae Min and Spencer M. Richards and Navid Azizan},
   booktitle = {Proc. IEEE Conf. Decis. Control},
   pages = {6032-6037},
   title = {Data-Driven Control with Inherent Lyapunov Stability},
   year = {2023}
}

@inproceedings{Gaagai2023,
   author = {Ramzi Gaagai and Mattia Giaccagli and Joachim Horn},
   booktitle = {Proc. IEEE Conf. Decis. Control.},
   pages = {7759-7766},
   title = {Distributed Leader-Followers Constrained Platooning Control of Linear Homogeneous Vehicles},
   year = {2023}
}

@inproceedings{Ploeg2011,
   author = {Jeroen Ploeg and Bart  Scheepers and Ellen Nunen and Nathan Wouw and Henk Nijmeijer},
   booktitle = {IEEE Intell. Transp. Syst. Conf. Proc.},
   pages = {260-265},
   title = {Design and experimental evaluation of cooperative adaptive cruise control},
   year = {2011}
}

@article{Li2023,
   author = {Yongming Li and Yongyan Zhao and Shaocheng Tong},
   number = {11},
   journal = {IEEE Trans. Fuzzy Syst.},
   month = nov,
   pages = {3934-3943},
   title = {Adaptive Fuzzy Control for Heterogeneous Vehicular Platoon Systems With Collision Avoidance and Connectivity Preservation},
   volume = {31},
   year = {2023}
}

@article{Degrave2022,
    author = {Jonas Degrave and others},
   number = {7897},
   journal = {Nature},
   month = feb,
   pages = {414-419},
   pmid = {35173339},
   publisher = {Nature Research},
   title = {Magnetic control of tokamak plasmas through deep reinforcement learning},
   volume = {602},
   year = {2022}
}

@article{Furieri2024,
   author = {Luca Furieri and Clara  Galimberti and Giancarlo Ferrari-Trecate},
   journal = {IEEE Open J. Control Syst.},
   pages = {342-357},
   title = {Learning to Boost the Performance of Stable Nonlinear Systems},
   volume = {3},
   year = {2024}
}

@article{Revay2024,
   author = {Max Revay and Ruigang Wang and Ian R. Manchester},
   number = {5},
   journal = {IEEE Trans. Autom. Control},
   month = may,
   pages = {2855-2870},
   title = {Recurrent Equilibrium Networks: Flexible Dynamic Models with Guaranteed Stability and Robustness},
   volume = {69},
   year = {2024}
}

@article{Mahfouz2023,
   author = {Dalia Mahfouz and Omar Shehata and Elsayed Morgan},
   number = {6},
   journal = {IEEE Trans. Intell. Transp. Syst},
   pages = {5685-5704},
   title = {Development and Evaluation of a Unified Integrated Platoon Control System Architecture},
   volume = {24},
   year = {2023}
}

@article{Besselink2017,
   author = {Bart Besselink and Karl Johansson},
   number = {9},
   journal = {IEEE Trans. Autom. Controll},
   pages = {4376-4391},
   title = {String Stability and a Delay-Based Spacing Policy for Vehicle Platoons Subject to Disturbances},
   volume = {62},
   year = {2017}
}

@article{Feng2019,
   author = {Shuo Feng and Yi Zhang and Shengbo Eben Li and Zhong Cao and Henry X. Liu and Li Li},
   journal = {Annu. Rev. Control},
   pages = {81-97},
   title = {String stability for vehicular platoon control: Definitions and analysis methods},
   volume = {47},
   year = {2019}
}

@article{Luo20212,
   author = {Qianyue Luo and Anh-Tu Nguyen and James Fleming and Hui Zhang},
   number = {4},
   journal = {IEEE Trans. Veh. Technol},
   month = apr,
   pages = {2930-2944},
   title = {Unknown Input Observer Based Approach for Distributed Tube-Based Model Predictive Control of Heterogeneous Vehicle Platoons},
   volume = {70},
   year = {2021}
}

@article{Nunen2019,
  author={van Nunen, Ellen and Reinders, Joey and Semsar-Kazerooni, Elham and van de Wouw, Nathan},
  journal={IEEE Trans. Intell. Veh.}, 
  title={String Stable Model Predictive Cooperative Adaptive Cruise Control for Heterogeneous Platoons}, 
  year={2019},
  volume={4},
  number={2},
  pages={186-196},
  }

\end{document}